\documentclass{article}

\usepackage{arxiv}

\usepackage[utf8]{inputenc} 
\usepackage[T1]{fontenc}    
\usepackage{nicefrac}       
\usepackage{microtype}      
\usepackage{amsmath}
\usepackage{amssymb}
\usepackage{amsfonts}
\usepackage{amsthm}         

\usepackage{graphicx}
\usepackage{booktabs}       
\usepackage{multirow}
\usepackage[table]{xcolor}
\usepackage[caption=false,font=normalsize,labelfont=sf,textfont=sf]{subfig}

\usepackage[ruled,vlined,linesnumbered]{algorithm2e}
\SetKwInput{KwInput}{Input}
\SetKwInput{KwOutput}{Output}
\SetKw{KwInit}{Initialize}
\SetKw{KwRet}{return}
\SetKw{KwAnd}{and}
\SetKw{KwOr}{or}

\usepackage{url}          
\usepackage{natbib}
\let\cite\citep
\usepackage{doi}
\usepackage{hyperref}       

\usepackage[shortcuts,acronym]{glossaries}
\glsdisablehyper
\newif\ifuniqueAffiliation
\uniqueAffiliationfalse

\usepackage{authblk}

\newbox{\orcid}\sbox{\orcid}{\includegraphics[scale=0.06]{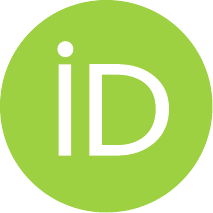}} 

\newtheorem{theorem}{Theorem}

\newtheorem{proposition}[theorem]{Proposition}

\newtheorem{definition}{Definition}

\newcommand{\bx}{\boldsymbol{x}}
\newcommand{\by}{\boldsymbol{y}}
\newcommand{\bK}{\boldsymbol{K}}

\newcommand{\bZ}{\boldsymbol{Z}}
\newcommand{\bG}{\boldsymbol{G}}
\newcommand{\bGR}{\boldsymbol{G}_R}

\newcommand{\balpha}{{\boldsymbol{\alpha}}}

\newacronym{ML}{ML}{machine learning}
\newacronym{DL}{DL}{deep learning}
\newacronym{RL}{RL}{representation learning}
\newacronym{RKHS}{RKHS}{Reproducing Kernel Hilbert Space}
\newacronym{MnL}{MnL}{manifold learning}
\newacronym{AE}{AE}{autoencoder}
\newacronym{KE}{KE}{Kernel Embedding}
\newacronym{PSD}{PSD}{Postitive Semidefinite}

\title{Learning Reconstructive Embeddings in Reproducing Kernel Hilbert Spaces via the Representer Theorem \thanks{This work has been submitted to the IEEE for possible publication}}

\date{}

\author[1]{%
    \href{https://orcid.org/0009-0005-5068-3166}{\usebox{\orcid}\hspace{1mm}Enrique Feito-Casares\thanks{\texttt{enrique.feito@urjc.es}}}%
}
\author[1]{%
    \href{https://orcid.org/0000-0001-6916-6082}{\usebox{\orcid}\hspace{1mm}Francisco M. Melgarejo-Meseguer}%
}
\author[1,2]{%
    \href{https://orcid.org/0000-0003-0426-8912}{\usebox{\orcid}\hspace{1mm}José-Luis Rojo-Álvarez}%
}

\affil[1]{Departamento de Teoría de la Señal y Comunicaciones y Sistemas Telemáticos y Computación, \protect\\ Universidad Rey Juan Carlos, Madrid, Spain.}
\affil[2]{Centro de Investigación for Data, Complex Networks and Cybersecurity Sciences,  \protect\\ Universidad Rey Juan Carlos, Madrid, Spain.}
\renewcommand{\shorttitle}{Learning Reconstructive Embeddings in RKHS}

\hypersetup{
pdftitle={Learning Reconstructive Embeddings in RKHS via the Representer Theorem},
pdfsubject={cs.LG, stat.ML},
pdfauthor={Feito-Casares et al.},
pdfkeywords={Reproducing Kernel Hilbert Space, Representer Theorem, manifold learning, kernel alignment},
}

\begin{document}
\maketitle

\begin{abstract}
Motivated by the growing interest in representation learning approaches that uncover the latent structure of high-dimensional data, this work proposes new algorithms for reconstruction-based manifold learning within Reproducing-Kernel Hilbert Spaces (RKHS). Each observation is first reconstructed as a linear combination of the other samples in the RKHS, by optimizing a vector form of the Representer Theorem for their autorepresentation property. A separable operator-valued kernel extends the formulation to vector-valued data while retaining the simplicity of a single scalar similarity function. A subsequent kernel-alignment task projects the data into a lower-dimensional latent space whose Gram matrix aims to match the high-dimensional reconstruction kernel, thus transferring the auto-reconstruction geometry of the RKHS to the embedding. Therefore, the proposed algorithms represent an extended approach to the autorepresentation property, exhibited by many natural data, by using and adapting well-known results of Kernel Learning Theory. Numerical experiments on both simulated (concentric circles and swiss-roll) and real (cancer molecular activity and IoT network intrusions) datasets provide empirical evidence of the practical effectiveness of the proposed approach.
\end{abstract}



\section{Introduction}

As experienced by any practitioner, data representation is critical to the application of \ac{ML}, whatever the targeted task, supervised or unsupervised. In this sense, recent efforts have focused their attention on \ac{RL} as an effective way to structure knowledge, enabling models to learn latent representations that capture the underlying patterns in the data and facilitate generalization. This trend has gained momentum with the rise of deep learning and unsupervised feature learning, where representations learned directly from raw data have demonstrated superior performance across a wide range of tasks \cite{Bengio2013}. Among the various approaches for generating representations, \ac{MnL} has attracted increasing attention as a powerful approach to \ac{RL} for producing compact, low-dimensional encodings of high-dimensional inputs. \ac{MnL} relies on the assumption that data lie on a lower-dimensional manifold embedded within a higher-dimensional space, and aims to uncover this latent structure by preserving the intrinsic geometry of the data. This approach provides meaningful representations that can be used to prepare data for downstream machine learning tasks.

Recent studies have highlighted the impact of incorporating \ac{MnL} techniques into both \ac{ML} and \ac{DL} pipelines to address a variety of domain-specific challenges. In biomedical applications, \ac{MnL} has been used for stress detection~\cite{Bodaghi2024}, electrocardiographic signal assessment~\cite{Sanchez-Carballo2025}, and ambient health monitoring in smart homes~\cite{Melgarejo-Meseguer2024a}. In engineering domains, it supports structural damage detection~\cite{Entezami2025}, aerodynamic design optimization~\cite{Ma2024}, and cybersecurity in critical systems~\cite{Feijoo-Martinez2023}. Moreover, the widespread use of \ac{MnL} for data visualization has been well documented, with techniques such as Student’s t-distributed Stochastic Neighbor Embedding (t-SNE) \cite{vanderMaaten2008} and Uniform Manifold Approximation and Projection (UMAP) \cite{McInnes2018a} facilitating the exploration and interpretation of data structures in lower-dimensional spaces \cite{Ali2019a,Li2024}.

Whether previous \ac{MnL} techniques tend to preserve distances, neighbourhood probabilities, and local linearity, the present work instead centers dimensionality reduction on the reconstruction geometry within a \ac{RKHS}, so the resulting embedding is shaped by data samples reconstruction ability rather than distance alone. To make this idea work for vector-valued data, the study introduces a separable operator-valued kernel that applies a single scalar kernel function uniformly across every output dimension. Finally, a kernel-alignment objective ties the high-dimensional reconstruction relationships to the low-dimensional embedding, ensuring that the learned latent space mirrors the structure uncovered in the original feature space. The contributions of the present work can be summarized as follows: (i) a new autoreconstructive criterion for manifold learning in RKHS, and (ii) an embedding mechanism that aligns reconstruction geometry with latent similarity. Both contributions are supported by numerical experiments on various datasets, including kernelized input data.

The paper is structured as follows. Section \ref{sec:related_work} reviews existing dimensionality reduction techniques to position the present work within this landscape. Sections \ref{sec:background_and_preliminaries} and \ref{sec:methodology} introduce the mathematical foundations and describe the proposed algorithm. Section \ref{sec:numerical_experiments} presents quantitative and qualitative results comparing our method to established baselines. Finally, Section \ref{sec:discussion_and_conclussion} discusses the results and outlines directions for future research.

\section{Related Work}
\label{sec:related_work}

In the context of the present proposal, several related works share a similar conceptual approach, particularly concerning dimensionality reduction tools and, more specifically, the use of Kernel Theory.

Dimensionality reduction techniques for high-dimensional data can be broadly classified into three paradigms, namely, spectral methods, probabilistic methods, and neural network-based methods \cite{Ghojogh2023a}. Each paradigm provides valuable insights into the latent structure of data, but also presents limitations that motivate our approach. Spectral methods, such as Principal Component Analysis (PCA) \cite{Jolliffe1986}, Fisher Discriminant Analysis (FDA) \cite{Hastie2009}, and their kernelized variants (e.g., KPCA \cite{Scholkopf1997} and KFDA \cite{Mika1999}), extract key directions through generalized eigenvalue problems. Probabilistic approaches, Probabilistic PCA (PPCA) \cite{Roweis1997}, SNE \cite{Hinton2002}, and UMAP \cite{McInnes2018a} provide uncertainty quantification and capture complex nonlinear relationships. However, these methods are limited by intractable inference problems and restrictive noise assumptions and rarely yield explicit mappings. Neural network-based techniques, particularly Autoencoders (AEs), have proven effective in learning rich, nonlinear representations. However, while they can achieve state-of-the-art performance, their latent representations often lack transparency, making them true black-box models. In line with the proposed approach, recent attempts have focused on kernelizing autoencoder criteria, initially in shallow autoencoders \cite{Gholami2016a}, and then extending this theoretical analysis to learning low-dimensional latent spaces in vector-valued \ac{RKHS}. This generalizes the kernelization criterion of autoencoders to an arbitrary number of layers, using a reconstruction criterion defined on the input data space, in which the dimensionality reduction occurs at the bottleneck defined by the neural network \cite{Laforgue2018,Laforgue2019a}.

The criterion for reconstructing data samples based on neighbors was originally proposed under the Local Linear Embedding (LLE) algorithm \cite{Roweis2000a}. Subsequent efforts led to the development of a kernelized variant (KLLE) \cite{Zhao2012a}, which performs reconstruction in \ac{RKHS}. While this method shares similarities with our approach, here instead, is considered self-reconstruction for the embedded data points, thus also focusing on the input space reconstruction simultaneously.

Kernel Dimensionality Reduction (KDR) \cite{Fukumizu2003} is a family of methods within the Sufficient Dimensionality Reduction (SDR) paradigm \cite{Li1991}. Despite the similarity in name, KDR is primarily a supervised statistical technique designed to reduce the covariance of data in the \ac{RKHS} and project it into a lower-dimensional space. An unsupervised variant has also been proposed \cite{Wang2010}; however, our methodology differs fundamentally from all existing variants of KDR.

\section{Background on Dimensionality Reduction and Kernel Methods}
\label{sec:background_and_preliminaries}

In this section, we first outline the theoretical foundations of dimensionality reduction, introducing the manifold hypothesis as the central motivation for many nonlinear approaches. We then present the basic concepts of kernel methods, including kernel functions, \ac{PSD} kernels, and feature maps, before extending them to operator-valued kernels. Building on these notions, we formalize the framework of \ac{RKHS} in both scalar and vector-valued settings, together with their associated norms and representer theorems. Furthermore, an experienced practitioner with kernel methods can choose to skip this section or complement it with general references on the subject~\cite{Scholkopf2002,Rojo-Alvarez2018c}.

Within the dimensionality reduction methodology, it is assumed that each feature of data points does not carry the same amount of information, meaning that the $D$-dimensional data points do not cover the entire $D$-dimensional Euclidean space ($\mathbb{R}^D$) \cite{Ghojogh2023a}. Whether this hypothesis is true must be proven; however, since data usually comes from physical measurements and underlying elementary explanations might be available, this statement is often formalized under the manifold hypothesis, which motivates learning algorithms that exploit this latent structure.
\begin{definition}[Manifold Hypothesis \cite{Fefferman2016}]
The manifold hypothesis asserts a set of high-dimensional data $\{\bx_i\}_{i=1}^n \subset \mathbb{R}^D$ can lie approximately on a smooth, low-dimensional manifold $\mathcal{M} \subset \mathbb{R}^D$ of intrinsic dimension $d \ll D$. 
\end{definition}
A majority of the introduced \ac{MnL} algorithms specifically aim to identify non-linear manifolds within the data. In this context, kernel methods \cite{Scholkopf2002} could provide a valuable complement, offering a principled mathematical framework for learning nonlinear relationships by embedding data into high-dimensional or even infinite-dimensional Hilbert spaces. This methodological framework relies on the concept of kernel function, which can be informally understood as a similarity measure between pairs of inputs. Formally, a kernel function is defined as follows:
\begin{definition}[Kernel Function]
A kernel function is a function $k: \mathcal{X} \times \mathcal{X} \to \mathbb{R}$ that quantifies similarity between elements $\bx, \bx' \in \mathcal{X}$. In the context of machine learning, $\mathcal{X}$ is typically a subset of $\mathbb{R}^D$, and $k(\bx, \bx')$ expresses how closely related $\bx$ and $\bx'$ are, according to some fixed notion of similarity.
\end{definition}
Not all kernel functions are suitable for constructing inner product spaces. For this, we require additional structure in terms of symmetry and \ac{PSD}. These properties ensure that a given kernel corresponds to an inner product in a Hilbert space.
\begin{definition}[Positive Semi-Definite Kernel]
A kernel function $k: \mathcal{X} \times \mathcal{X} \to \mathbb{R}$ is called \ac{PSD} if it is symmetric (i.e., $k(\bx,\bx') = k(\bx', \bx)$ for all $\bx,\bx' \in \mathcal{X}$), and for every finite subset $\{\bx_1, \dots, \bx_n\} \subset \mathcal{X}$, the corresponding kernel matrix $\bK \in \mathbb{R}^{n \times n}$, defined by
\begin{equation}
\bK_{ij} = k(\bx_i, \bx_j),
\end{equation}
is \ac{PSD}, i.e., $\boldsymbol{v}^\top \bK \boldsymbol{v} \geq 0$ for all $\boldsymbol{v} \in \mathbb{R}^n$.
\end{definition}
The abstraction provided by kernel functions when employed on a given dataset is materialized in matrix form as a kernel matrix (or Gram matrix) encoding inner products between mapped data points.
An alternative approach to \ac{PSD} kernels is through a feature map, which expresses each kernel as an inner product in an associated Hilbert space.
\begin{definition}[Feature Map]
Given a kernel function $k$, a feature map is a function $\phi: \mathcal{X} \to \mathcal{H}$ into a Hilbert space $\mathcal{H}$ such that
\begin{equation}
\bK(\bx, \bx') = \langle \phi(\bx), \phi(\bx') \rangle_{\mathcal{H}} \quad \text{for all } \bx, \bx' \in \mathcal{X}.
\end{equation}
\end{definition}
In the multi-output regression cases, where the relationship between inputs and outputs is more complex and better modeled through operators acting on a Hilbert space \cite{Micchelli2005}. The kernel methodology paradigm is equivalently defined for vector-valued inputs and outputs.
\begin{definition}[Operator-Valued Kernel Function]
An operator-valued kernel function $ k: \mathcal{X} \times \mathcal{X} \to \mathcal{L}(\mathcal{H}) $ is a function that quantifies the similarity between elements $ \bx, \bx' \in \mathcal{X} $ in terms of operators. Here, $ \mathcal{X} \subseteq \mathbb{R}^D $, and the kernel function returns an operator $ k(\bx, \bx') $ in a Hilbert space $ \mathcal{H} $, typically a bounded linear operator in $ \mathcal{L}(\mathcal{H}) $, which represents how closely related the inputs $\bx$ and $ \bx' $ are, based on a fixed notion of similarity.
\end{definition}
\begin{definition}[Operator-Valued Kernel Matrix]
Given a dataset $ \{\bx_1, \dots, \bx_n\} \subset \mathcal{X} $ and an operator-valued kernel function $ k $, the kernel matrix $ \bK \in \mathcal{L}(\mathcal{H})^{n \times n} $ is defined by $ \bK_{ij} = k(\bx_i, \bx_j) $.
\end{definition}
Given the provided definitions, an RKHS formalizes the connection between kernels, feature maps, and function spaces and provides a rich and rigorous setting for both analysis and algorithm design. First, the classical scalar-valued case is presented before generalizing to vector-valued outputs.
\begin{definition}[Scalar-Valued RKHS]
Let $k: \mathcal{X} \times \mathcal{X} \to \mathbb{R}$ be a \ac{PSD} kernel. The \ac{RKHS} $\mathcal{H}_k$ associated with $k$ is a Hilbert space of functions $f: \mathcal{X} \to \mathbb{R}$ such that:
\begin{enumerate}
    \item For each $\bx \in \mathcal{X}$, the function $k(\bx, \cdot)$ belongs to $\mathcal{H}_k$.
    \item The reproducing property holds:
    \begin{equation}
    f(\bx) = \langle f, k(\bx, \cdot) \rangle_{\mathcal{H}_k}, \quad \forall f \in \mathcal{H}_k, \; \bx \in \mathcal{X}.
    \end{equation}
\end{enumerate}
\end{definition}
In many applications, the learning task involves predicting structured or vector-valued outputs (e.g., $\mathbb{R}^m$-valued responses). To accommodate these settings, the scalar RKHS framework must be extended to handle vector-valued functions. This leads to the theory of vector-valued RKHSs \cite{Senkene1973}.
\begin{definition}[Vector-Valued RKHS]
Let $\mathcal{Y}$ be a Hilbert space, typically $\mathbb{R}^m$. A vector-valued RKHS $\mathcal{H}_K$ is a Hilbert space of functions $f: \mathcal{X} \to \mathcal{Y}$ associated with an operator-valued kernel $K: \mathcal{X} \times \mathcal{X} \to \mathcal{L}(\mathcal{Y})$, such that:
\begin{enumerate}
    \item For each $\bx \in \mathcal{X}$ and $\by \in \mathcal{Y}$, the function $K(\bx, \cdot)\by$ belongs to $\mathcal{H}_K$.
    \item The reproducing property holds:
    \begin{equation}
    \langle f(\bx), \by \rangle_{\mathcal{Y}} = \langle f, K(\bx, \cdot)\by \rangle_{\mathcal{H}_K}, \quad \forall f \in \mathcal{H}_K, \; \by \in \mathcal{Y}.
    \end{equation}
\end{enumerate}
\end{definition}
Both scalar and vector-valued RKHSs admit a natural norm structure that plays a critical role in regularization and model complexity control. These norms are derived from the inner product structure and, in practice, appear in closed form when functions are expressed in terms of kernel evaluations.
\begin{definition}[RKHS Norms]
Let $f \in \mathcal{H}_k$ be a scalar-valued function represented as
\begin{equation}
f(\bx) = \sum_{i=1}^n \alpha_i k(\bx_i, \bx), \quad \alpha_i \in \mathbb{R},
\end{equation}
Then the RKHS norm is given by
\begin{equation}
\|f\|_{\mathcal{H}_k}^2 = \sum_{i,j=1}^n \alpha_i \alpha_j k(\bx_i, \bx_j).
\end{equation}
In the vector-valued case, if $f(\bx) = \sum_{i=1}^n k(\bx_i, \bx)\alpha_i$ with $\alpha_i \in \mathcal{Y}$, the norm in $\mathcal{H}_K$ is given by
\begin{equation}
\|f\|_{\mathcal{H}_K}^2 = \sum_{i,j=1}^n \langle \alpha_i, K(\bx_i, \bx_j) \alpha_j \rangle_{\mathcal{Y}}.
\end{equation}
\end{definition}
The following theorem states that the solution to a wide class of optimization problems in a scalar-valued RKHS lies in the finite-dimensional span of kernel evaluations at the training points. This result is the foundation of kernel methods in supervised learning.
\begin{theorem}[Representer Theorem - Scalar-Valued Case]
Let $\mathcal{H}_k$ be a scalar-valued RKHS with kernel $k: \mathcal{X} \times \mathcal{X} \to \mathbb{R}$. Given training data $(\bx_i, y_i) \in \mathcal{X} \times \mathbb{R}$ for $i=1, \dots, n$, the solution to the regularized empirical risk minimization problem
\begin{equation}
\min_{f \in \mathcal{H}_k} \sum_{i=1}^n \mathcal{L}(y_i, f(\bx_i)) + \lambda \|f\|_{\mathcal{H}_k}^2
\end{equation}
admits a representation of the form
\begin{equation}
f(\bx) = \sum_{i=1}^n \alpha_i k(\bx_i, \bx), \quad \alpha_i \in \mathbb{R}.
\end{equation}
\end{theorem}
In the vector-valued setting, a similar result holds, where the coefficients become vectors and the scalar kernel is replaced by an operator-valued kernel. This generalization is critical for problems involving multi-output prediction and structured outputs.
\begin{theorem}[Representer Theorem - Vector-Valued Case]
Let $\mathcal{H}_K$ be a vector-valued RKHS with kernel $K: \mathcal{X} \times \mathcal{X} \to \mathcal{L}(\mathcal{Y})$. Given training data $(\bx_i, \by_i) \in \mathcal{X} \times \mathcal{Y}$ for $i=1, \dots, n$, the solution to the regularized problem
\begin{equation}
\min_{f \in \mathcal{H}_K} \sum_{i=1}^n \mathcal{L}(\by_i, f(\bx_i)) + \lambda \|f\|_{\mathcal{H}_K}^2
\end{equation}
admits a representation of the form
\begin{equation}
f(\bx) = \sum_{i=1}^n k(\bx_i, \bx) \boldsymbol{\alpha_i}, \quad \boldsymbol{\alpha_i} \in \mathcal{Y}.
\end{equation}
\end{theorem}
To conclude, we highlight two widely used kernels: the linear and Gaussian kernels. These serve as canonical examples at opposite ends of the kernel spectrum.
\begin{definition}[Linear Kernel]
The linear kernel is defined as
\begin{equation}
K(\bx, \bx') = \langle \bx, \bx' \rangle_{\mathbb{R}^D}.
\end{equation}
\end{definition}

\begin{definition}[Gaussian Kernel]
The Gaussian kernel (also known as the RBF kernel) is given by
\begin{equation}
K(\bx, \bx') = \exp\left( -\frac{\|\bx - \bx'\|^2}{2\sigma^2} \right),
\end{equation}
where $\sigma > 0$ is a bandwidth parameter. This kernel defines an infinite-dimensional RKHS and is universal on compact domains.
\end{definition}

In summary, the concepts presented in this section provide the mathematical framework necessary for the development of the proposed methodology. In the following section, these theoretical tools are applied to formulate the autoreconstructive kernel embedding algorithm.

\begin{figure*}[ht]
    \centering
    \includegraphics[width=\textwidth]{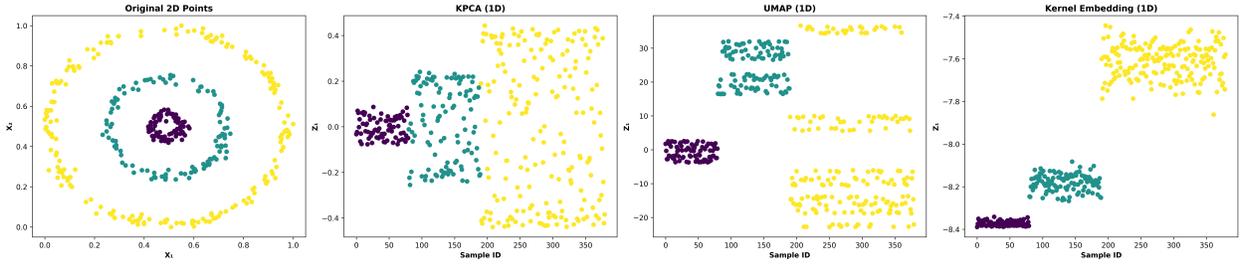}
    \caption{
    Embeddings of the concentric circles dataset. From left to right: original 2D data, Kernel PCA with 1D output, UMAP with 1D output, and the proposed \ac{KE} with 1D output. Each point is colored by its cluster label. The horizontal axis for the 1D methods corresponds to the sample index.
    }
    \label{fig:cc_embedding_comparison}
\end{figure*}

\section{Autoreconstructive Kernel Embedding}
\label{sec:methodology}

This section presents the proposed algorithmic framework for RKHS-based dimensionality reduction, referred to as Autoreconstructive Kernel Embedding (abbreviated as \ac{KE}). The method proceeds in two stages: First, reconstruction weights are optimized to encode local geometry in the RKHS (A); Second, a low-dimensional embedding is learned by aligning its kernel structure with the original reconstruction kernel (B). Each step is supported by theoretical results, including a closed-form solution for the reconstruction coefficients and an alignment-based loss for the embedding.

\subsection{Reconstruction Weight Learning}
This phase estimates a coefficient vector that reconstructs each mapped data point in the \ac{RKHS} as a weighted combination of its neighboring points. The coefficients are chosen to minimize the overall reconstruction error, thereby encoding the local geometric structure of the data in feature space.

Let $\mathcal{X} = \{ \bx_1, \bx_2, \dots, \bx_n \} \subset \mathbb{R}^D$ be a dataset consisting of $n$ samples. Each data point $\bx_i$ is mapped into a \ac{RKHS} $\mathcal{H}$ through a feature map $\phi : \mathbb{R}^d \to \mathcal{H}$, associated with a PSD kernel $k : \mathcal{X} \times \mathcal{X} \to \mathbb{R}$. The goal is to reconstruct each point $\phi(\bx_i)$ in feature space using a linear combination of other points $\phi(\bx_j)$, without explicitly accessing $\phi$ itself. The objective is to minimize the total reconstruction error in RKHS, which can be expressed as

\begin{equation}
\min_{\boldsymbol{\beta}} \sum_{i=1}^{n} \|\phi(\bx_i) - f_{\phi}(\bx_i)\|^2_\mathcal{H}
\label{eq:rkhs-reconstruction}
\end{equation}
where $f_{\phi}$ is the reconstruction function in the \ac{RKHS}.
To characterize the solution to the minimization problem in Eq.~\eqref{eq:rkhs-reconstruction}, recall that the Representer Theorem constrains the form of the optimal reconstruction function in an RKHS. The following proposition states this result explicitly.
\begin{proposition}[Kernel Expansion of the Reconstruction Function]
The reconstruction function $f_{\phi}(\bx)$ can be expressed as a kernel expansion in a vector-valued \ac{RKHS}. By the Representer Theorem, the optimal function lies in the span of $\{\phi(\bx_j)\}_{j=1}^n$ and is of the form:

\begin{equation}
f_{\phi}(\bx) = \sum_{i=1}^n \beta_i G(\bx, \bx_i) \phi(\bx_i),
\end{equation}
where $G(\bx, \bx_i)$ is a kernel function, and the coefficients $\beta_i$ are the learned reconstruction weights.
The kernel used in the reconstruction framework is assumed to be separable. Specifically, it can be written as the product of a scalar kernel function $G(\bx, \bx')$ that measures the similarity between input points and the identity operator $I_{\mathcal{H}}$ that acts on the vector-valued RKHS space. This separability allows the kernel to define pairwise similarities between input points while acting uniformly across all directions of the feature space. Formally, the kernel takes the form
\begin{equation}
K(\bx, \bx') = G(\bx, \bx') \cdot I_{\mathcal{H}},
\end{equation}
where the scalar part $G(\bx, \bx')$ captures the pairwise similarity between data points $\bx$ and $\bx'$, and this similarity is used to express relationships in the input space $\mathbb{R}^d$ and the vector part, represented by the identity operator $I_{\mathcal{H}}$, ensures that the kernel operates uniformly across all dimensions of the RKHS.

\end{proposition}
\begin{proof}
By the Representer Theorem, any solution to the regularized empirical risk minimization problem
\begin{equation}
\min_{f \in \mathcal{H}_K} \sum_{i=1}^{n} \|\phi(\bx_i) - f(\bx_i)\|_{\mathcal{H}}^2 + \lambda \|f\|_{\mathcal{H}_K}^2
\end{equation}
must be of the form 
\begin{equation}
f(\bx) = \sum_{i=1}^n k(\bx_i, \bx) \boldsymbol{\balpha}_i,
\end{equation}
where $\boldsymbol{\alpha_i} \in \mathcal{H}$. In this case, we use the separable kernel $K(\bx, \bx') = G(\bx, \bx') I_{\mathcal{H}}$ and represent the coefficients $\boldsymbol{\alpha_i}$ as $\boldsymbol{\alpha_i} = \beta_i \phi(\bx_i)$, yielding the kernel expansion:
\begin{equation}
f_{\phi}(\bx) = \sum_{i=1}^n \beta_i G(\bx, \bx_i) \phi(\bx_i).
\end{equation}
Thus, the reconstruction function $f_{\phi}$ lies in the span of $\{\phi(\bx_i)\}$, as required by the Representer Theorem.
\end{proof}
Given the proposed minimization problem in the RKHS, a central insight from Kernel Learning Theory is that it is unnecessary to work directly with the mapping functions $\phi(\cdot)$. This is because such problems can be expressed entirely in terms of inner products, which are already computed via the kernel matrix. The following propositions reformulate the optimization problem to derive a loss function expressed in terms of the kernel matrix.
\begin{proposition}[Optimization of the reconstruction weights in an RKHS]
Fix a PSD kernel $G$ with Gram matrix $\bG=G(\bx_i,\bx_j)$ and separate its diagonal,
\begin{equation}
  \tilde \bG \;=\; \bG-\operatorname{diag}(\bG),
  \qquad
  (\;\tilde G_{ii}=0\text{ for every }i\;).
\end{equation}
For a coefficient vector
$\boldsymbol\beta=(\beta_1,\dots,\beta_n)^{\!\top}\in\mathbb R^{n}$ put
$\boldsymbol{D}=\operatorname{diag}(\boldsymbol\beta)$ and define the linear
reconstruction operator of $f_{\phi}(\bx)$
\begin{equation}
  \boldsymbol{F}
    \;=\;
    \tilde \bG\,\boldsymbol{D}\,\boldsymbol{\Phi}
    \;=\;
    \bigl[\,
      \sum_{j\ne1} \beta_j G_{1j}\phi(\bx_j),
      \;\dots\;,
      \sum_{j\ne n}\beta_j G_{nj}\phi(\bx_j)
    \bigr],
\end{equation}
where $\boldsymbol{\Phi}=[\phi(\bx_1) \:,\: \dots \:,\: \phi(\bx_n)]\in\mathcal H^{\,n}$. Then the empirical reconstruction error
\begin{equation}
  \mathcal L(\boldsymbol\beta)
      \;=\;\bigl\|\boldsymbol{\Phi}-\boldsymbol{F}\bigr\|_{F}^{2}
      \;=\;
      \sum_{i=1}^{n}
      \Bigl\|
         \phi(\bx_i)
         -\!\!\sum_{\substack{j=1\\j\neq i}}^{n}\!\!
              \beta_j\,G(\bx_i,\bx_j)\,\phi(\bx_j)
      \Bigr\|^{2}
\end{equation}
admits the quadratic representation
\begin{equation}\label{eq:quad-form}
\;
     \mathcal L(\boldsymbol\beta)
       \;=\;
       \boldsymbol\beta^{\!\top}\boldsymbol A\,\boldsymbol\beta
       \;-\;2\,\boldsymbol c^{\!\top}\boldsymbol\beta
       \;+\;\operatorname{tr}(\bG)
  \;
\end{equation}
with
\begin{equation}\label{eq:A-c}
  \boldsymbol A \;=\; \bG \circ \bigl(\tilde \bG^{2}\bigr),
  \qquad
  c_i \;=\; \sum_{j\neq i} G_{ij}^{2}
                 \;=\; (G\tilde G)_{ii},
\end{equation}
where “$\circ$’’ denotes the Hadamard (element-wise).
\end{proposition}

\begin{proof}
Because $\boldsymbol\Phi^{\!\top}\boldsymbol \Phi=\bG$ and $\tilde \bG$ is symmetric,
\begin{equation}
  \boldsymbol \Phi-\boldsymbol F
  \;=\;\boldsymbol{\Phi-\tilde G D\Phi}
  \;=\;\boldsymbol{(I-\tilde G D)\,\Phi},
\end{equation}
and the Frobenius norm becomes
\begin{equation}
  \mathcal L(\boldsymbol\beta)
  \;=\;
  \bigl\|\boldsymbol{(I-\tilde G D)\Phi}\bigr\|_{F}^{2}
  \;=\;
  \operatorname{tr}\!\bigl[\boldsymbol{(I-\tilde G D)\,G\,(I-D\tilde G)}\bigr].
\end{equation}
Expanding the trace and using the cyclic property together with the
diagonality of~$\boldsymbol D$, one obtains
\begin{equation}
\begin{aligned}
  \mathcal L(\boldsymbol\beta)
  &=\operatorname{tr}(\bG)
     -2\,\operatorname{diag}(\bG \tilde \bG)^{\!\top}\boldsymbol\beta
     +\sum_{j,\ell}\beta_j\beta_\ell
        G_{j\ell}\!(\tilde \bG^{2})_{j\ell} \\
  &=\operatorname{tr}(\bG)
     -2\,\boldsymbol c^{\!\top}\boldsymbol\beta
     +\boldsymbol\beta^{\!\top}
        \!\bigl(\bG\circ\tilde \bG^{2}\bigr)\boldsymbol\beta,
\end{aligned}
\end{equation}
which is Eq.\eqref{eq:quad-form} with the definitions in Eq.\eqref{eq:A-c}.
\end{proof}

Having expressed the \ac{RKHS} reconstruction loss in quadratic form, the next step is to determine the coefficient vector $\boldsymbol{\beta}$ that minimizes this loss. The following proposition establishes the conditions under which a closed-form solution exists and provides the explicit expression for the minimizer.

\begin{proposition}[Closed-Form Minimizer of the Reconstruction Loss]
If the matrix $\boldsymbol{A} = \boldsymbol{G} \circ (\tilde{\boldsymbol{G}}^2)$ is nonsingular, the minimizer of $\mathcal{L}(\boldsymbol{\beta})$ is unique and given by
\begin{equation}
\boldsymbol{\beta}^\star = \boldsymbol{A}^{-1} \boldsymbol{c}.
\end{equation}
If $\boldsymbol{A}$ is PSD but singular, the minimum-norm solution is
\begin{equation}
\boldsymbol{\beta}^\star = \boldsymbol{A}^{\dagger} \boldsymbol{c},
\label{eq:beta-star}
\end{equation}
where ${}^\dagger$ denotes the Moore–Penrose pseudoinverse.
\end{proposition}

\begin{proof}
From the quadratic form in Proposition~\ref{eq:quad-form}, the gradient with respect to $\boldsymbol{\beta}$ is
\begin{equation}
\nabla_{\boldsymbol{\beta}} \mathcal{L} = 2 \boldsymbol{A} \, \boldsymbol{\beta} - 2 \boldsymbol{c}.
\end{equation}
Setting this to zero yields the normal equations $\boldsymbol{A} \boldsymbol{\beta} = \boldsymbol{c}$.  
If $\boldsymbol{A}$ is invertible, the solution is $\boldsymbol{A}^{-1} \boldsymbol{c}$ and strict convexity ensures uniqueness.  
If $\boldsymbol{A}$ is singular but PSD, the set of minimizers is $\{ \boldsymbol{\beta} \,|\, \boldsymbol{A} \boldsymbol{\beta} = \boldsymbol{c} \}$; the pseudoinverse yields the minimum-norm element in this set.
\end{proof}

The above propositions complete the first stage of the method by providing an explicit kernel-based formulation for computing the optimal reconstruction weights $\boldsymbol{\beta}$. These weights encode the local geometric relationships between points in the \ac{RKHS} and serve as the foundation for the second stage, where is learned a low-dimensional embedding whose kernel structure aligns with the reconstruction geometry defined in this step.

\begin{figure*}[ht]
    \centering
    \includegraphics[width=\textwidth]{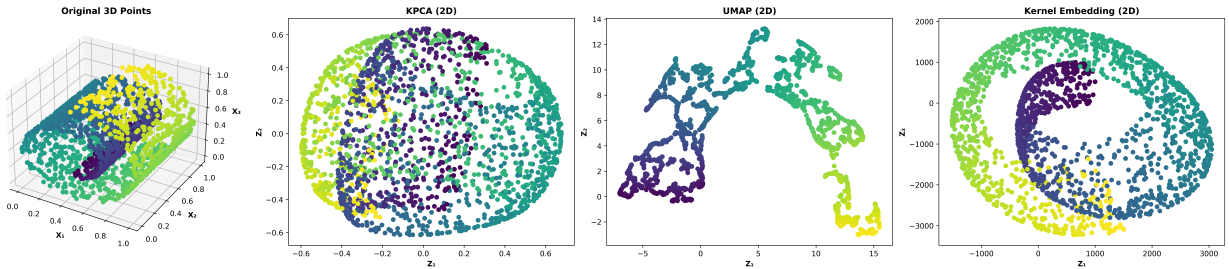}
    \caption{
    Embeddings of the Swiss roll dataset. From left to right: original 3D data, Kernel PCA with 2D output, UMAP with 2D output, and the \ac{KE} with 2D output. Points are colored based on their position along the manifold. The 3D plot shows the original data geometry; the remaining plots show the corresponding 2D projections.
    }
    \label{fig:swiss_roll_embedding_comparison}
\end{figure*}
\subsection{Embedding via Kernel Alignment}

With the reconstruction weights $\boldsymbol{\beta}$ computed in the first step, this second step aims to learn a low-dimensional embedding that preserves the autoreconstructive geometry of the \ac{RKHS} learned previously. The central idea is to construct an embedding whose induced kernel matrix matches, as closely as possible, the reconstruction kernel $\bGR$ obtained from the high-dimensional feature space. The following proposition formalizes this alignment objective. 

\begin{proposition}[Kernelized Reconstruction Objective for the Embedding Step]
Let $\bG_R\in\mathbb{R}^{n\times n}$ be the reconstruction kernel from the RKHS stage, and let $\boldsymbol{\beta}\in\mathbb{R}^n$ be the fixed reconstruction weights with $D=\mathrm{diag}(\boldsymbol{\beta})$. For a target dimension $d$, parameterize the embedding as
\begin{equation}
\boldsymbol{Z} = \bG_R\,\balpha \in \mathbb{R}^{n\times d},
\label{eq: z-embedding}
\end{equation}
with learnable $\balpha\in\mathbb{R}^{n\times d}$. 
Let $\bG_H\in\mathbb{R}^{n\times n}$ be the kernel matrix computed from the embedded points $\boldsymbol{Z}$, and define
\begin{equation}
\tilde{\bG}_H \;=\; \bG_H - \mathrm{diag}(\bG_H),
\end{equation}
\begin{equation}
A(\bG_H)\;=\;\bG_H \circ (\tilde{\bG}_H^2),
\end{equation}
\begin{equation}
c(\bG_H)\;=\;\operatorname{diag}(\bG_H \tilde{\bG}_H),
\end{equation}
where $\circ$ denotes the Hadamard product.  
Then the latent RKHS reconstruction error induced by the operator $\tilde{\bG}_H D$ is
\begin{equation}
\mathcal{L}_\balpha(\balpha,\boldsymbol{\beta})\;=\;\boldsymbol{\beta}^\top A(\bG_H)\,\boldsymbol{\beta} \;-\; 2\,c(\bG_H)^\top \boldsymbol{\beta} \;+\; \operatorname{tr}(\bG_H).
\end{equation}
Minimizing $\mathcal{L}_\balpha$ over $\balpha$ aligns the embedding-space kernel $\bG_H$ with the original reconstruction kernel $\bG_R$.
\end{proposition} 
Although $\mathcal{L}_\balpha(\balpha;\boldsymbol{\beta})$ has closed quadratic form it is not quadratic in $\balpha$, because $\bG_H$ depends on $\balpha$ through $\bZ=\bG_R\balpha$ and the (typically nonlinear) kernel on the embedded points. Consequently, there is no closed‑form minimizer for $\balpha$; one must optimize $\mathcal{L}_\balpha(\balpha;\boldsymbol\beta)$ by gradient‑based iterative methods , backpropagating through $\bZ=\bG_R\balpha$ and $\bG_H$.
Both phases share the same kernelized quadratic template in $\boldsymbol{\beta}$ and this reuse of $\mathcal{L}(\cdot,\boldsymbol{\beta})$ ensures that the latent kernel $\bG_H(\balpha)$ is trained to exhibit the same autoreconstructive geometry as enforced in step (A).

\subsection{Learning Algorithm}

With the previous propositions, the proposed algorithm can be summarized as a two–phase procedure that progressively transfers autoreconstructive geometry from the high–dimensional \ac{RKHS} to a low-dimensional embedding.
\begin{enumerate}
    \item  Reconstruction Weight Learning. In the first phase, we determine the reconstruction weights $\boldsymbol{\beta}$ that best express each feature–space point $\phi(\bx_i)$ as a linear combination of its neighbors.  This is formulated as
    \begin{equation}
    \min_{\boldsymbol{\beta}}
    \sum_{i=1}^n
    \left\|
      \phi(\bx_i) -
      \sum_{\substack{j=1\\j\neq i}}^{n} \beta_j\, G_R(\bx_i, \bx_j)\, \phi(\bx_j)
    \right\|_{\mathcal{H}}^2,
    \end{equation}
where $G_R(\bx_i,\bx_j)$ is the reconstruction kernel, i.e., the Gram matrix of the original data in the RKHS with its diagonal removed. This objective reduces to a convex quadratic form in $\beta$ with closed–form solution.
\item Learning Low–Dimensional Embeddings.
With $\boldsymbol{\beta}$ fixed, we seek a latent projection $\balpha \in \mathbb{R}^{n \times d}$ such that the embedding
\begin{equation}
\bZ \;=\; \bG_R\,\balpha
\end{equation}
induces a latent kernel $\bG_H$ that replicates the same reconstruction pattern as in previous step. Since $\bG_H$ depends nonlinearly on $\balpha$ through both $\bZ$ and the kernel map, no closed–form solution exists, and gradient–based optimization is used to minimized iteratively.
\end{enumerate}
Together, these two phases constitute the Autoreconstructive Kernel Embedding procedure: local reconstruction relations are first encoded in the RKHS and then consistently transferred to the latent space through a kernel alignment objective

\subsection{Out–of–Sample Embedding}
Once the low-dimensional mapping $ \bZ = \bG_R \balpha $ is learned on the training data, we extend the embedding to unseen test points using Nystrom approximation. Let $\mathcal{X}^{\star} = \{ \bx'_1, \bx'_2, \dots, \bx'_m \} \subset \mathbb{R}^D$ be a test dataset consisting of $m$ samples and $ \bG_R^{\star} \in \mathbb{R}^{m \times n}$ denote the reconstruction kernel between test and training points, computed as
\begin{equation}
\bG_R^{\star} = 
\begin{bmatrix}
k(\bx'_1, \bx_1) & \cdots & k(\bx'_1, \bx_n) \\
\vdots & \ddots & \vdots \\
k(\bx'_m, \bx_1) & \cdots & k(\bx'_m, \bx_n)
\end{bmatrix}
\end{equation}
 Given the learned projection matrix $\balpha \in \mathbb{R}^{n \times d}$, a batch of new points is embedded through the kernel similarity matrix
\begin{equation}
\bZ' = \bG_R^{\star} \balpha.
\end{equation}
This formulation respects the same RKHS-induced geometry learned during training and provides a principled way to embed new samples without retraining, maintaining consistency with the autoreconstructive structure.

\subsection{Computational Complexity Analysis}
The following discussion details the time complexity of the proposed method. The Memory requirements of the algorithm, notably the $O(n^2)$ storage cost of the kernel matrix, are not analyzed in depth. Conserving the previous notation, $n$ denotes the number of data points, $d$ represents the latent dimensionality used in the computation of $\balpha$, and $T$ indicates the number of gradient descent iterations in the learning low-dimensional embedding step. Each operation of the algorithm is described below:

\begin{enumerate}
    \item Kernel matrix computation \\
    Construction of the Gram matrix $\bK$: $\mathcal{O}(n^2)$.
    \item Computation of $\tilde{\bG}$ (zero-diagonal form) \\
    Removal and reinsertion of $n$ diagonal entries: $\mathcal{O}(n^2)$.
    \item Computation of $\boldsymbol{A}$ matrix (Eq.\ref{eq:A-c}) \\
    Matrix-matrix product ${\bG}\tilde{\bG}$: $\mathcal{O}(n^3)$ \\
    Hadamard product $\boldsymbol{A} = \bG \circ (\tilde{\bG}^2)$: $\mathcal{O}(n^2)$ \\
    Asymptotic cost: $\mathcal{O}(n^3)$.

    \item Computation of vector $\boldsymbol{c}$ (Eq.\ref{eq:A-c}) \\
    Matrix-matrix product $\bG \tilde{\bG}$: $\mathcal{O}(n^3)$ \\
    Diagonal extraction: $\mathcal{O}(n)$ \\
    Asymptotic cost: $\mathcal{O}(n^3)$.

    \item Quadratic loss evaluation (Eq. \ref{eq:quad-form}) \\
    Evaluation of $\boldsymbol{\beta}^{\top} \boldsymbol{A} \boldsymbol{\beta} - 2 \boldsymbol{c}^{\top} \boldsymbol{\beta} + \operatorname{tr}(\bK)$: $\mathcal{O}(n^2)$.

    \item Closed-form solver (Eq. \ref{eq:beta-star}) \\
    Dense linear solve $\boldsymbol{A} \boldsymbol{\beta} = \boldsymbol{c}$: $\mathcal{O}(n^3)$.

    \item Computation of $\boldsymbol{Z}$ matrix (Eq. \ref{eq: z-embedding}) \\
    Product of $d \times n$ matrix and $n \times n$ kernel: $\mathcal{O}(d \, n^2)$.

    \item Gradient descent solver ($\balpha$-mode) \\
    Per iteration:
    \begin{itemize}
        \item Backward pass through the loss: $\mathcal{O}(n^2)$
        \item Loss computation (as in Step 5): $\mathcal{O}(n^2)$
    \end{itemize}
\end{enumerate}
The computational complexity can be divided into two phases. In the reconstruction weights learning stage ($\beta$-mode), the dominant cost arises from Steps 1-6, leading to a total complexity of $\mathcal{O}(n^3)$. In the embedding kernel alignment stage optimization requires $T$ gradient descent iterations, each involving Steps 1-6 for kernel recomputation, which results in a total complexity of $\mathcal{O}(T \, n^3)$. These estimates are derived under the standard dense-matrix model and assume naïve cubic costs for general matrix-matrix products and dense linear solves. In practice, the use of optimized routines and PyTorch kernels significantly reduces wall-clock time for many operations.

The proposed methodology offers a straightforward approach to produce embeddings that preserve autoreconstructive properties in the \ac{RKHS}. The following section delves into the numerical experiments conducted with this algorithm.

\section{Numerical Experiments}
\label{sec:numerical_experiments}
This section presents a series of numerical experiments conducted to evaluate the proposed \ac{RKHS} dimensionality reduction framework. The method is assessed on both synthetic and real-world datasets, with particular attention to embedding fidelity and classification performance.

For unsupervised tasks, we evaluated the quality of the low-dimensional embeddings using two standard clustering validity indices: the Davies-Bouldin Index \cite{Davies1979} and the Calinski-Harabasz Index \cite{Calinski1974}, and two embedding fidelity measurements: trustworthiness and continuity \cite{vanderMaaten2008,Venna2001} . The Davies-Bouldin Index quantifies the average similarity between each cluster and its most similar neighbor; lower values indicate more compact and well-separated clusters.

All experiments were conducted on a high-performance computing server equipped with an NVIDIA A40 GPU (46,068 MiB memory) and dual AMD EPYC 7763 64-core processors, totaling 128 logical CPUs. Training and evaluation scripts were implemented in Python 3.11.11, and released in the accompanying repository \cite{feito_casares_2025_16812806}.

Two baseline methods were used for comparison: KPCA and UMAP. These methods collectively represent classical, non-linear kernel methods and newer probabilistic MnL techniques.

\begin{figure*}[ht]
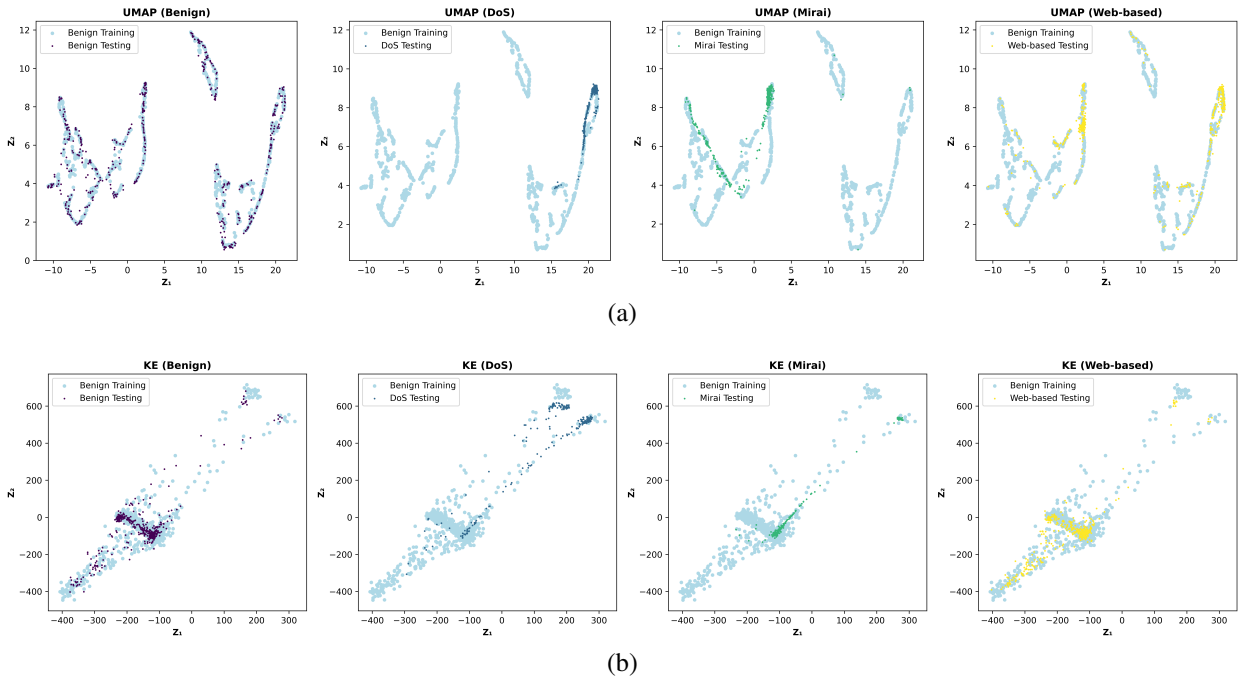

    \centering
    \begin{minipage}{\textwidth}
        \centering
        \includegraphics[width=\textwidth]{figures/CICCIOT_2023_UMAP.png}
        (a)
    \end{minipage}
    
    \vspace{1em}

    \begin{minipage}{\textwidth}
        \centering
        \includegraphics[width=\textwidth]{figures/CICCIOT_2023_KE.png}
        (b)
    \end{minipage}

    \caption{
    Embeddings on the CIC-CIoT2023 dataset. Both visualizations display benign training samples (light blue) and test samples from four traffic categories: Benign, DoS, Mirai, and Web-based. Panel (a) embeddings produced by the UMAP algorithm, while panel (b) embeddings obtained using the proposed \ac{KE} method.}
    \label{fig:comparison_umap_ke}
\end{figure*}
\subsection{Concentric Circles}

The concentric circles dataset consists of three interleaved circular clusters in $\mathbb{R}^2$. This synthetic dataset is designed to evaluate the ability of a dimensionality reduction method to preserve nonlinear separability in a compact latent representation.

Figure~\ref{fig:cc_embedding_comparison} shows the embeddings obtained using three baseline methods and the proposed \ac{KE}. In this experiment, each method compresses the original 2D input to a one-dimensional latent space. Both Kernel PCA and UMAP fail to produce a meaningful linear ordering that separates all three classes, resulting in partially or completely overlapping clusters. In contrast, the proposed method clearly separates the three groups along the 1D axis, demonstrating its ability to preserve a nonlinearly separable structure under strong dimensionality reduction. Additionally, as shown in Table~\ref{tab:clustering_indices}, the \ac{KE} significantly outperforms the other methods on both metrics. Its Davies-Bouldin score is nearly two orders of magnitude lower than the next best method (UMAP), and its Calinski-Harabasz score is substantially higher, indicating that the resulting embedding yields compact, well-separated clusters. These results confirm the capacity of the algorithms to preserve nonlinear structures in low-dimensional spaces and validate the qualitative observations made from the visualizations.

\begin{table}[ht]
\centering
\caption{Clustering validity indices for different representations. Lower Davies-Bouldin and higher Calinski-Harabasz values indicate better-defined clusters.}
\label{tab:clustering_indices}
\begin{tabular}{ccc}
\hline
\textbf{Representation} & \textbf{Davies-Bouldin} ↓ & \textbf{Calinski-Harabasz} ↑ \\
\hline
Original Space          & 23.30                     & 0.395                         \\
\hline
KPCA (2D)                & 45.29                     & 0.410                         \\
\hline
UMAP (2D)               & 9.37                      & 142.21                        \\
\hline
\ac{KE} (2D)   & \textbf{0.21}             & \textbf{5887.58}              \\
\hline
\end{tabular}
\end{table}

\subsection{Swiss Roll}
The Swiss roll dataset is a canonical three-dimensional manifold with intrinsic curvature, commonly used to evaluate the ability of dimensionality reduction algorithms to preserve global geometry and nonlinear structure. The usual MnL goal is to project the curled 3D surface into a two-dimensional space while maintaining the relative ordering of points along the manifold.

Figure~\ref{fig:swiss_roll_embedding_comparison} illustrates different embeddings generated from the Swiss roll data with the different methods. The first panel shows the original 3D manifold, where the smooth spiral structure of the roll is evident. Kernel PCA, shown in the second panel, projects the data into a 2D space but fails to fully unroll the manifold, resulting in a projection where neighboring points may become disordered. UMAP, shown in the third panel, succeeds in unfolding the roll but introduces noticeable local distortions and compressions in some regions, particularly toward the boundaries. The last panel displays the embedding produced by the proposed method. Visually, it reveals a smooth and continuous unfolding of the manifold with a clear preservation of the point order along the intrinsic dimension. The color gradient-applied based on the geodesic position of each point-remains consistent and continuous throughout the embedding, suggesting strong fidelity to the original manifold structure.

To complement the visual analysis, we also measured the ability of each method to preserve the local neighborhood structure using two standard metrics, trustworthiness and continuity. Trustworthiness assesses the degree to which spurious neighbors are introduced in the embedding, while continuity evaluates how many true neighbors from the original space are retained. The results, summarized in Table~\ref{tab:trust_cont}, show that while UMAP achieves nearly perfect trustworthiness, it suffers from low continuity. In contrast, the \ac{KE} balances both metrics more effectively, suggesting better consistency in neighborhood preservation across scales.

\begin{table}[ht]
\centering
\caption{Trustworthiness and continuity scores for each dimensionality reduction method ($k=15$). Higher values indicate better local structure preservation.}
\label{tab:trust_cont}
\begin{tabular}{lcc}
\hline
\textbf{Method} & \textbf{Trustworthiness ↑ } & \textbf{Continuity ↑ } \\
\hline
KPCA              & 0.8617 & 0.7496 \\
UMAP              & \textbf{0.9990} & 0.1769 \\
\ac{KE}  & 0.9061 & \textbf{0.6118} \\
\hline
\end{tabular}
\end{table}

\subsection{Cancer Biomolecules}

\begin{table*}[t]
\centering
\caption{Classification performance of the \ac{KE}  method across different latent dimensions. Tanimoto kernel baselines using SVM and max-margin conditional random
field (MMCRF) are included for comparison. Arrows indicate that higher values are better. Best results per column are highlighted in bold.}
\label{tab:ke_classification_comparison}
\begin{tabular}{cccccc}
\hline
\textbf{Latent Dim / Method} & \textbf{Sensitivity ↑} & \textbf{Specificity ↑} & \textbf{Accuracy ↑} & \textbf{ROC AUC ↑} & \textbf{F1 Score ↑} \\
\hline
KE-2     & 0.330 ± 0.125 & 0.842 ± 0.142 & 0.671 ± 0.046 & 0.620 ± 0.020 & 0.408 ± 0.084 \\
KE-3     & 0.190 ± 0.211 & 0.918 ± 0.236 & 0.678 ± 0.057 & 0.649 ± 0.021 & 0.262 ± 0.127 \\
KE-5     & 0.283 ± 0.164 & 0.910 ± 0.195 & 0.702 ± 0.049 & 0.697 ± 0.023 & 0.387 ± 0.092 \\
KE-10    & 0.354 ± 0.146 & 0.903 ± 0.180 & 0.719 ± 0.052 & 0.710 ± 0.024 & 0.463 ± 0.088 \\
KE-20    & \textbf{0.441 ± 0.133} & 0.901 ± 0.178 & \textbf{0.749 ± 0.054} & \textbf{0.758 ± 0.024} & 0.550 ± 0.077 \\
KE-50    & 0.231 ± 0.204 & 0.918 ± 0.237 & 0.692 ± 0.057 & 0.699 ± 0.023 & 0.320 ± 0.112 \\
KE-100   & 0.120 ± 0.237 & \textbf{0.927 ± 0.249} & 0.660 ± 0.060 & 0.529 ± 0.006 & 0.147 ± 0.158 \\
\hline
SVM \cite{su2010multilabel}   & -- & -- & 0.641 & -- & 0.527 \\
MMCRF \cite{su2010multilabel} & -- & -- & 0.676 & -- & \textbf{0.562} \\
\hline
\end{tabular}
\end{table*}

This set of experiments was conducted on a real-world dataset of molecular fingerprints from the U.S. National Cancer Institute (NCI) \cite{su2010multilabel}. Each molecule is represented as a high-dimensional binary vector, indicating the presence or absence of specific chemical substructures. These types of molecular descriptors are widely used in drug discovery and cancer compound screening tasks. Given the sparsity and binary nature of the data, molecular similarities were computed using the Tanimoto kernel, an established choice in cheminformatics for comparing binary vectors.

To assess the ability of the proposed method to extract meaningful structure, we first applied it to generate low-dimensional embeddings of molecular fingerprint data. To evaluate how performance varies with the degree of compression, we conducted experiments across a range of latent dimensions, from 2 to 100. For each embedding, we trained a linear Support Vector Machine (SVM) using five-fold stratified cross-validation to predict cancer-type activity, and evaluated performance using sensitivity, specificity, accuracy, ROC (Receiver Operating Characteristic), AUC (Area Under the Curve), and F1 score. 

As shown in Table \ref{tab:ke_classification_comparison}, the \ac{KE} method achieves its best overall performance at 20 dimensions (KE-20), reaching a sensitivity of 0.441, specificity of 0.901, accuracy of 0.749, and a peak ROC AUC of 0.758. These values represent the optimal trade-off among the measured criteria. Even under extreme compression (e.g., KE-2 and KE-3), the method performs competitively, achieving accuracies of 0.671 and 0.678, respectively higher than the classical SVM baseline (0.641) and comparable to MMCRF (0.676). The F1 score also peaks at KE-20 (0.550), only slightly below MMCRF 0.562. However, performance begins to degrade at higher dimensions, particularly in sensitivity and F1 score, suggesting potential overfitting or diminished embedding efficiency. Although MMCRF achieves the best F1 score, the proposed method surpasses it in accuracy and ROC AUC, indicating a stronger overall discriminative capability. This highlights the distinctiveness of our method, which achieves competitive results while simultaneously learning compact, informative representations. Nonetheless, a full comparison remains limited, as prior studies did not report all the evaluation metrics.

\subsection{IoT Network Intrusions}

This last set of experiments was conducted on the CICIoT2023 dataset, a comprehensive real-world collection of IoT network traffic data containing both benign and malicious activities across multiple attack categories. The dataset encompasses diverse cyber threats, including DDoS attacks, DoS attacks, spoofing techniques, reconnaissance activities, web-based attacks, Mirai botnet variants, and brute force attempts \cite{Neto2023a}. Each network flow is characterized by a rich set of 46 features spanning flow statistics, protocol flags, and packet counts.

Given the heterogeneous nature of network traffic features and the multi-class classification challenge in cybersecurity, a custom multi-domain kernel was designed to capture domain-specific similarities across statistical properties, protocol behaviors, traffic patterns, and geometric flow characteristics, providing a more nuanced representation for IoT network intrusion detection. A key strength of our RKHS-based approach lies in its flexibility to incorporate domain knowledge directly into the kernel desing. To assess performance, we trained both UMAP and our \ac{KE}  model using only benign traffic. The goal is to evaluate the generalization ability to unseen attacks without any exposure to malicious labels. As shown in Figure~\ref{fig:comparison_umap_ke}, UMAP (a) yields a fragmented latent space characterized by multiple dispersed branches, without establishing a centralized or coherent region for benign traffic. In contrast, the proposed \ac{KE} method (b) constructs a compact and well-structured manifold, where benign samples concentrate in a dominant region, and malicious samples are projected into distinct, more localized areas that remain well-separated from the benign core. Importantly, this structured organization emerges even though \ac{KE} was trained solely on benign traffic, underscoring its ability to generalize to novel threats by leveraging the inductive bias encoded in the kernel design and the autoreconstructive properties of the algorithm.

These numerical experiments support the effectiveness of the proposed autoreconstructive kernel embedding method across synthetic and real-world datasets. On tasks like concentric circles and Swiss roll, it outperforms Kernel PCA and UMAP in both clustering quality and embedding fidelity. These results highlight its capacity to preserve nonlinear structures under strong compression. In the molecular classification task, the method achieves competitive accuracy and ROC AUC while producing compact, informative embeddings. Overall, the approach offers a principled solution for structure-preserving dimensionality reduction, suitable for both visualization and downstream analysis.

\section{Discussion and Conclussion}
\label{sec:discussion_and_conclussion}

The proposed Autoreconstructive Kernel Embedding framework introduces a principled reconstruction-driven perspective for learning low-dimensional embeddings that retain the self-reconstructive geometry of high-dimensional data.  A key strength of the method lies in its modular two-step design, which decouples reconstruction from embedding in a principled approach based on the Representer Theorem and the autorepresentation property as novelties. This allows the preservation of nonlinear similarity structures while yielding embeddings that are robust and effective for downstream tasks such as clustering and classification. The method generalizes well across domains with structured, high-dimensional, and possibly sparse data provided that an appropriate similarity kernel is available.

Nevertheless, several limitations remain. First, the method lacks an explicit inverse mapping from the embedding space to the input domain, which complicates generative applications. Second, out-of-sample extension requires additional approximation schemes, as the embedding is defined only for training samples. Finally, although our formulation is unsupervised, incorporating supervision example, via label-informed reconstruction weights or alignment penalties-could enhance its applicability to discriminative tasks.

Future work will also address scalability to larger datasets through sparse kernel approximations and explore supervised extensions and online versions of the algorithm. We also envision applications in areas such as anomaly detection, drug discovery, and explainable AI.


\section*{Acknowledgments}
This work was supported by the CyberFold project, funded by the European Union through the NextGenerationEU instrument (Recovery, Transformation, and Resilience Plan), and managed by Instituto Nacional de Ciberseguridad de España (INCIBE), under reference number ETD202300129. Partially funded by the Autonomous Community of Madrid (ELLIS Madrid Node). Also partially supported by project PID2022-140786NB-C32 (LATENTIA) and PID2023-152331OA-I00 (HERMES) from the Spanish Ministry of Science and Innovation (AEI/10.13039/501100011033).

\bibliographystyle{unsrtnat}
\bibliography{mybibfile}

\end{document}